\documentclass[11pt]{article}
\usepackage[margin=1in]{geometry}
\usepackage{amsmath,amssymb,amsthm}
\usepackage{booktabs}
\usepackage{enumitem}
\usepackage{graphicx}
\usepackage[round,authoryear]{natbib}
\usepackage{setspace}
\usepackage{hyperref}
\usepackage[capitalise,nameinlink]{cleveref}

\newtheorem{theorem}{Theorem}
\newtheorem{lemma}{Lemma}
\newtheorem{corollary}{Corollary}
\newtheorem{assumption}{Assumption}
\crefname{assumption}{Assumption}{Assumptions}
\Crefname{assumption}{Assumption}{Assumptions}

\title{Regime-Calibrated Fleet Repositioning with a Spatial Queue-Regret Decomposition}
\author{
  Indar Kumar \and
  Akanksha Tiwari
}
\date{May 2026}

\begin{document}
\onehalfspacing
\maketitle

\begin{abstract}
Ride-hailing and autonomous mobility-on-demand operators reposition idle supply
before future demand is fully observed. We study a retrieval-calibrated
predict-then-optimize approach for this problem: historical demand regimes are
matched to the current query block, combined into a calibrated demand prior,
and passed to a fleet-balancing controller. The paper makes three contributions.
First, we train a leakage-safe similarity gate whose objective penalizes demand
error, pickup spatial mismatch, and queue shortage risk rather than retrieval
rank alone. Second, we develop a spatial queue-regret decomposition for a stable
queueing surrogate, linking demand-field error to wait through queueing
sensitivity, allocator sensitivity, and Wasserstein pickup mismatch. Third, we
evaluate learned retrieval and external-style rebalancing baselines in a common
simulator. In the calibrated-demand gate experiment, across eight New York City
scenarios and ten seeds, the spatial gate reduces mean wait to 82.3s, compared
with 85.3s for hand-tuned similarity and 85.8s for a distributional-only
baseline. In a separate replay-demand controller comparison, a scenario
chance-MPC analog and a share-target transportation LP improve on Wen-style
rebalancing (92.2s/92.2s vs. 100.1s), a reduced GPR chance-MPC comparator is
intermediate at 94.4s, and an oracle MPC diagnostic is 91.3s.
\end{abstract}


\section{Introduction}
\label{sec:intro-v2}

Fleet repositioning is a recurring operational problem in ride-hailing,
delivery, and mobility-on-demand systems. Operators must decide where idle
vehicles should wait before the next wave of requests materializes. Forecasting
alone is insufficient: the forecast has value only through the downstream
matching and repositioning decision it induces, a central theme in
predict-then-optimize and prescriptive-analytics systems
~\citep{bertsimas2020predictive,elmachtoub2022smart}.

This paper studies a regime-calibrated approach. Given a query block, the method
retrieves similar historical demand regimes, constructs a calibrated prior over
time and space, and uses that prior in an optimization-based fleet balancing
controller. The emphasis is methodological: what estimation errors matter for
repositioning, how similarity weights can be learned without regime leakage, and
how the resulting controller compares with external-style baselines in one
simulator.

The contributions are:
\begin{enumerate}[leftmargin=*]
  \item A leakage-safe learned similarity gate over regime-matching components,
  selected by downstream wait rather than retrieval rank alone.
  \item A spatial queue-regret decomposition for a stable queueing surrogate,
  linking demand estimation error, allocation sensitivity, pickup Wasserstein
  mismatch, and expected wait.
  \item Simulator-consistent external-style baselines, including Wen-style
  historical rebalancing, share-target LP, scenario and reduced GPR-style
  chance-MPC analogs, time-boxed contextual DQN, and oracle MPC.
\end{enumerate}

\begin{figure}[t]
\centering
\includegraphics[width=0.92\linewidth]{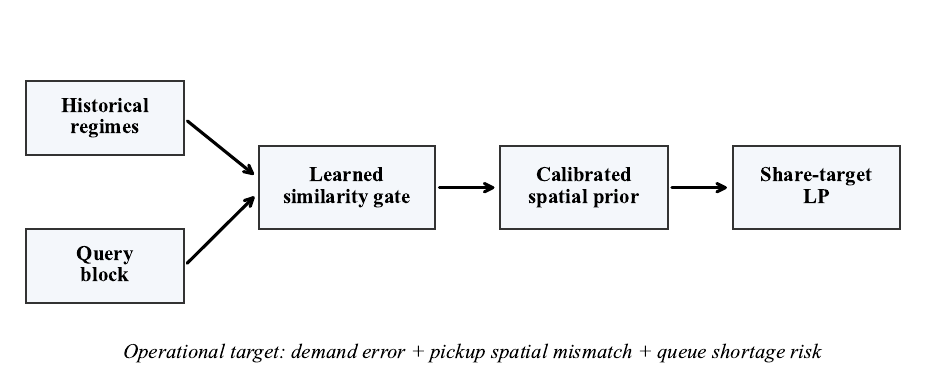}
\caption{Method overview: learned regime retrieval constructs a calibrated
spatial prior that drives a share-target transportation LP.}
\label{fig:method-v2}
\end{figure}

\section{Related Work}
\label{sec:related-v2}

\textbf{Fleet repositioning and matching.}
Online matching and dispatch for ride-hailing systems have been studied through
dynamic assignment, batching, and stochastic matching models
~\citep{lowalekar2018online,bei2018algorithms,alonsoMora2017ondemand}. Fleet
repositioning work commonly balances idle supply against predicted demand using
fluid, optimization, or MPC-style controllers
~\citep{spieser2016shared,wen2017rebalancing,iglesias2018data,wallar2018vehicle}.
Foundational AMoD analyses use closed queueing networks and network-flow
formulations to characterize rebalancing and fleet sizing
~\citep{pavone2012robotic,zhang2016control,iglesias2019bcmp}. Our share-target
LP follows this optimization tradition but is evaluated with the same simulator
and demand blocks as the learned calibration method.
Closest in spirit is recent smart predict-then-optimize work for vehicle
rebalancing, which trains a forecasting model against downstream decision
regret in a regional MILP framework~\citep{guo2026smart}. Our focus is
different: we learn a leakage-safe retrieval gate over historical regimes and
use spatial queue-regret terms to define the retrieval target.

\textbf{Learning for fleet operations.}
Deep RL has been applied to large-scale fleet management and order dispatch
~\citep{lin2018efficient,tang2019deep}. Our experiments include a contextual-DQN
warm start and a prior PPO negative result, but we do not position RL as the
paper's contribution. The main result comes from retrieval-calibrated
prediction plus deterministic optimization.

\textbf{Demand forecasting and regime retrieval.}
Spatio-temporal demand forecasting is often modeled with deep sequence or graph
architectures~\citep{yao2018deep,li2018diffusion}. We instead ask a narrower
operational question: which historical regimes should be retrieved as a prior
for repositioning? Regime models and Wasserstein comparisons motivate our
historical-block matching and spatial mismatch diagnostics
~\citep{hamilton1989new,ramdas2017wasserstein,kumar2026rg}.

\section{Problem Setting}
\label{sec:problem-v2}

The city is partitioned into zones $\mathcal{Z}=\{1,\ldots,Z\}$. During a
finite horizon, requests arrive with non-stationary spatial intensity
$\lambda^\star(t)$. At each repositioning epoch, the operator observes recent
demand, pending requests, and idle vehicle locations, but not future arrivals.
The decision is a feasible movement plan for idle vehicles. The objective is to
reduce rider wait while preserving completion rate and avoiding excessive empty
travel.

The information structure is predict-then-optimize. A calibrated prior
$\hat{\lambda}$ is constructed from historical regimes before the repositioning
LP is solved. The evaluation fixes scenario blocks and random seeds, so methods
differ only in the prior and controller they provide to the same simulator.

\section{Regime-Calibrated Method}
\label{sec:method-v2}

\subsection{Regime Library}

Historical trips are split into four-hour regime blocks. Each block stores a
five-minute demand series, event annotations, temporal context, summary
features, and a sampled pickup/dropoff spatial pool. A query block is compared
with library blocks using six similarity components: Kolmogorov--Smirnov,
Wasserstein, summary-feature distance, variance similarity, event-pattern
similarity, and temporal proximity.

\subsection{Learned Similarity Gate}

A hand-tuned version used static similarity weights. The revised method trains a
gating MLP
\[
  g_\theta(x_q)=\mathrm{softmax}(\mathrm{MLP}(x_q))\in\Delta^6
\]
from query features to similarity weights. The split is leakage-safe: all
historical blocks sharing the evaluation month, day type, and hour range are
held out of gate training. The initial pairwise gate improved retrieval rank
but worsened downstream wait, which showed that rank correlation is not enough.
The final gate uses a spatially regularized objective combining demand
$\ell_1$ error, pickup spatial transport mismatch, and stylized queue gap.
This design choice is empirical as well as theoretical: an initial gate trained
only for retrieval rank improved held-out rank correlation but worsened
simulator wait, whereas the spatially regularized gate is selected by the
frozen downstream validation protocol.

\subsection{Share-Target LP Controller}

For replay-demand external baselines, the main deterministic controller is a
share-target transportation LP. Given a demand-share prior $p$, current idle
supply $s$, and total idle count $I$, the controller targets
$I p_z$ idle vehicles in each zone and solves a min-cost transportation problem
from surplus zones to deficit zones. This reformulation is more stable than
comparing idle supply directly with raw near-term demand counts when total
future demand exceeds movable idle supply.
We also include two uncertainty-aware analogs: a scenario chance-MPC controller
that treats retrieved regimes as empirical demand scenarios, and a reduced GPR
chance-MPC comparator that smooths per-zone residual uncertainty before solving
the same transportation LP. These are simulator-native comparators rather than
full reproductions of any specific chance-constrained MPC implementation.


\section{Spatial Queue-Regret Theory}
\label{sec:theory-v2}

This section formalizes why the calibration problem is operational rather than
purely predictive. A demand estimate can be close in aggregate volume yet still
place supply in the wrong zones. The bound below separates temporal-volume
error, allocation instability, and spatial pickup mismatch.

The result is a surrogate decomposition for one stable planning horizon. It is
not a proof that the full simulator is globally stable or dynamically optimal
under repeated closed-loop decisions. Its purpose is to identify the operational
error terms a retrieval gate should penalize. Closed-loop behavior is validated
empirically in \Cref{sec:results-v2}, where the simulator includes finite
fleets, stochastic demand realizations, batch matching, travel times, and
repeated repositioning.

Let $\mathcal{Z}=\{1,\ldots,Z\}$ be the zone partition for one planning horizon.
The true arrival-rate vector is
$\lambda^\star\in\mathbb{R}_+^Z$ and the calibrated estimate is
$\hat{\lambda}\in\mathbb{R}_+^Z$. Let
$N=\sum_z\lambda^\star_z>0$. A repositioning allocator maps an estimate to an
integer capacity vector $A(\hat{\lambda})\in\mathbb{Z}_+^Z$, where
$A_z(\hat{\lambda})$ is the number of effective servers allocated to zone $z$.
For analysis, zone $z$ is approximated by an $M/M/c$ queue with service rate
$\mu$ per server, using the standard Erlang-C approximation for many-server
service systems~\citep{gross2008queueing,whitt2002stochastic}. Queueing models
are also common in mobility-on-demand and ride-share platform analyses
~\citep{banerjee2015pricing,pavone2012robotic,zhang2016control,iglesias2019bcmp}. Let
$W_z(c,\lambda)$ denote the Erlang-C mean queueing wait for zone $z$ and define
the demand-weighted system wait
\[
  \bar{W}(c,\lambda)
  =
  \sum_{z=1}^Z
  \frac{\lambda_z}{\sum_j \lambda_j}
  W_z(c_z,\lambda_z).
\]
Let $p^\star=\lambda^\star/N$ and
$\hat{p}=\hat{\lambda}/\sum_j\hat{\lambda}_j$ denote spatial demand
distributions. The spatial distance
$W_1^d(\hat{p},p^\star)$ is the 1-Wasserstein distance under a pickup-time
ground metric $d$.

\begin{assumption}[Stable operating set]
\label{ass:stable-set}
There is a fixed $\rho_{\max}<1$ such that all allocations compared in this
section satisfy
\[
  \lambda_z/(A_z(\lambda)\mu)\leq \rho_{\max}
  \qquad\text{for every active zone }z.
\]
Zones with zero arrival rate are omitted from the weighted sum.
\end{assumption}

\begin{assumption}[Finite active capacity set]
\label{ass:finite-capacity}
For the horizon and fleet size under study, active zone capacities lie in a
finite set $\{1,\ldots,c_{\max}\}$. This is automatic for a finite fleet and a
finite zone partition.
\end{assumption}

\begin{assumption}[Local allocator sensitivity]
\label{ass:allocator-sensitivity}
For the demand region induced by the calibration procedure, the allocator has
finite local $\ell_1$ sensitivity:
\[
  \|A(\hat{\lambda})-A(\lambda^\star)\|_1
  \leq L_A\|\hat{\lambda}-\lambda^\star\|_1 .
\]
This is the standard right-hand-side perturbation condition for a locally
nondegenerate transportation LP; when the active basis is unchanged, $L_A$ is
the induced norm of the basis inverse
~\citep{bertsimas1997linear,ahuja1993network}.
\end{assumption}

\begin{assumption}[Spatial dispatch Lipschitzness]
\label{ass:spatial-lipschitz}
For the dispatch policy and fleet scale under study, moving one unit of pickup
probability mass by distance $d$ changes demand-weighted wait by at most
$\beta d$. Equivalently, the pickup-location cost-to-wait map is
$\beta$-Lipschitz under the ground metric $d$.
\end{assumption}

\begin{lemma}[Erlang-C Sensitivity]
\label{lem:erlang-lipschitz}
Under \Cref{ass:stable-set,ass:finite-capacity}, there is a finite constant
$L_Q=L_Q(\mu,\rho_{\max},c_{\max})$ such that, for fixed capacity $c_z$,
\[
  |W_z(c_z,\lambda_z)-W_z(c_z,\lambda'_z)|
  \leq
  L_Q |\lambda_z-\lambda'_z|.
\]
Moreover, $L_Q$ scales no worse than
$O\!\left(1/[\mu(1-\rho_{\max})^2]\right)$ on the stable set.
\end{lemma}

\begin{proof}
For fixed $c_z$, the Erlang-C mean queueing wait can be written as
\[
  W_z(c_z,\lambda_z)
  =
  \frac{C(c_z,\lambda_z/\mu)}{c_z\mu-\lambda_z},
\]
where $C(\cdot)\in[0,1]$ is the delay probability. On the stable set,
$c_z\mu-\lambda_z\geq c_z\mu(1-\rho_{\max})$. Differentiating the rational
term gives a denominator of order $(c_z\mu-\lambda_z)^2$, and the derivative of
the bounded delay-probability term is finite on every compact subset of
$\rho<1$. Because active capacities are finite by \Cref{ass:finite-capacity},
we may take the maximum derivative over
$c_z\in\{1,\ldots,c_{\max}\}$ and
$\lambda_z/(c_z\mu)\leq\rho_{\max}$. The mean-value theorem gives the claim,
and the denominator term gives the stated heavy-traffic scaling.
\end{proof}

\begin{theorem}[Spatial Queue-Regret Bound]
\label{thm:spatial-queue-regret}
Suppose
\Cref{ass:stable-set,ass:finite-capacity,ass:allocator-sensitivity,ass:spatial-lipschitz}
hold. Then the operational regret from using $\hat{\lambda}$ instead of
$\lambda^\star$ satisfies
\[
\begin{aligned}
&\bar{W}(A(\hat{\lambda}),\lambda^\star)
 - \bar{W}(A(\lambda^\star),\lambda^\star) \\
&\quad\leq
  (1+\mu L_A)L_Q
  \frac{\|\hat{\lambda}-\lambda^\star\|_1}{N}
  + \beta W_1^d(\hat{p},p^\star).
\end{aligned}
\]
\end{theorem}

\begin{proof}
Add and subtract $\bar{W}(A(\hat{\lambda}),\hat{\lambda})$:
\[
\begin{aligned}
\bar{W}(A(\hat{\lambda}),\lambda^\star)
-\bar{W}(A(\lambda^\star),\lambda^\star)
&=
\underbrace{
\bar{W}(A(\hat{\lambda}),\lambda^\star)
-\bar{W}(A(\hat{\lambda}),\hat{\lambda})
}_{T_1} \\
&\quad+
\underbrace{
\bar{W}(A(\hat{\lambda}),\hat{\lambda})
-\bar{W}(A(\lambda^\star),\lambda^\star)
}_{T_2}.
\end{aligned}
\]
By \Cref{lem:erlang-lipschitz}, the direct demand perturbation term satisfies
$T_1\leq L_Q\|\hat{\lambda}-\lambda^\star\|_1/N$ after demand weighting.
For $T_2$, the capacity vector changes by
$\|A(\hat{\lambda})-A(\lambda^\star)\|_1$, which is at most
$L_A\|\hat{\lambda}-\lambda^\star\|_1$ by
\Cref{ass:allocator-sensitivity}. Each unit of server-capacity perturbation
corresponds to $\mu$ units of effective service-rate perturbation, so another
application of the queue sensitivity bound contributes
$\mu L_A L_Q\|\hat{\lambda}-\lambda^\star\|_1/N$.

The remaining error is spatial: even with the same aggregate volume, a
calibrated prior can move pickup mass to the wrong locations. By
\Cref{ass:spatial-lipschitz} and Kantorovich--Rubinstein duality for the
1-Wasserstein metric~\citep{villani2009optimal,ramdas2017wasserstein}, this
pickup misplacement contributes at most $\beta W_1^d(\hat{p},p^\star)$. Combining terms and
collecting the two queueing terms gives the result.
\end{proof}

\begin{corollary}[Directional Shortage]
\label{cor:directional-shortage}
If the calibrated and true demand vectors have equal total volume, then
\[
  \sum_z(\lambda^\star_z-\hat{\lambda}_z)_+
  =
  \frac{1}{2}\|\hat{\lambda}-\lambda^\star\|_1.
\]
Thus the queueing term in \Cref{thm:spatial-queue-regret} is proportional to
the under-forecast mass that can create local supply shortages.
\end{corollary}

\begin{corollary}[Retrieval-to-Wait Chain]
\label{cor:retrieval-to-wait}
Let $\hat{\lambda}_k=\sum_{i=1}^k\alpha_i\lambda_i$ be a similarity-weighted
mixture of retrieved regimes. If
$\|\lambda_i-\lambda^\star\|_1\leq\epsilon_i$ and
$W_1^d(p_i,p^\star)\leq\delta_i$, then
\[
  \mathcal{R}(\hat{\lambda}_k)
  \leq
  C_Q\sum_{i=1}^k\alpha_i\epsilon_i/N
  + \beta\sum_{i=1}^k\alpha_i\delta_i,
\]
where
$C_Q=(1+\mu L_A)L_Q$.
\end{corollary}

The corollary motivates the final learned-gate objective: the useful retrieval
criterion is not rank correlation alone, but a spatially weighted queue-regret
functional that penalizes demand error, pickup displacement, and shortage risk.

\begin{figure}[t]
\centering
\includegraphics[width=0.72\linewidth]{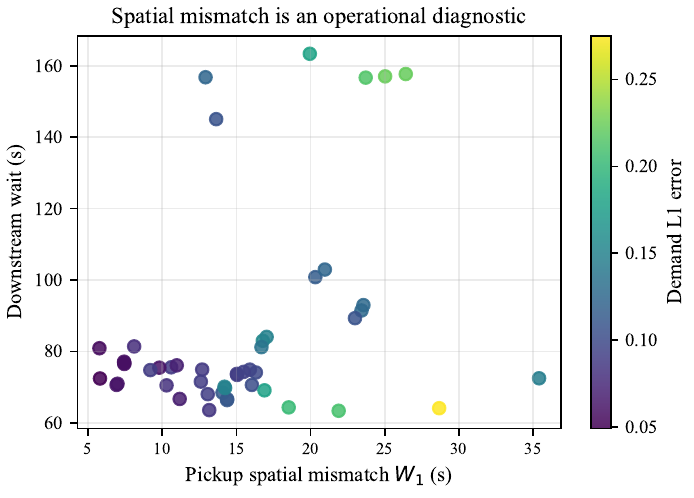}
\caption{Numerical diagnostic for the theory section. Pickup spatial mismatch
is an operational error term, not just a distributional statistic.}
\label{fig:theory-diagnostic-v2}
\end{figure}

\section{Experimental Protocol}
\label{sec:experiments-v2}

The main experiments use New York City TLC trips from January and June 2024.
The regime library contains 373 four-hour blocks with temporal, event, and
sampled spatial pickup/dropoff information. The leakage-safe split removes all
training blocks sharing the evaluation month, day type, and hour range, leaving
262 blocks for gate training and 104 held-out blocks for retrieval validation.

\begin{table}[t]
\centering
\caption{Evaluation scenarios.}
\label{tab:scenarios-v2}
\begin{tabular}{lll}
\toprule
Scenario & Calendar type & Time block \\
\midrule
jan\_nye\_am & holiday & morning \\
jan\_nye\_pm & holiday & evening \\
jan\_weekday\_am & weekday & morning \\
jan\_weekday\_pm & weekday & evening \\
jan\_weekend\_mid & weekend & midday \\
jun\_late\_night & summer weekday & late night \\
jun\_weekday\_am & summer weekday & morning \\
jun\_weekday\_pm & summer weekday & evening \\
\bottomrule
\end{tabular}
\end{table}

Main gate and external-baseline results use seeds 42--51. The primary metric is
mean rider wait. Secondary metrics are completion rate, tail wait, pickup
distance, idle time, and runtime. Reported significance tests are paired across
the 80 scenario-seed observations. The learned-gate experiment evaluates
calibrated synthetic demand streams and asks whether learned similarity improves
downstream wait. The external-baseline experiment evaluates controllers on
replay demand and asks whether the rebalanced LP improves on historical-share
rebalancing under identical simulator conditions.

The two comparisons intentionally answer different questions. In the gate
experiment, replay demand is a reference for the calibrated demand stream. In
the external-baseline experiment, replay demand is the common demand stream for
all controllers. This separation prevents controller performance from being
confounded with synthetic demand generation.

\section{Results}
\label{sec:results-v2}


\begin{table}[t]
\centering
\caption{Learned gate validation over 8 scenarios and 10 seeds.}
\label{tab:gate-v2}
\setlength{\tabcolsep}{4pt}
\begin{tabular}{lrrr}
\toprule
Method & Wait (s) & Completion & Vs. replay (\%) \\
\midrule
Spatial learned gate & 82.3 & 0.957 & 32.3 \\
Hand-tuned similarity & 85.3 & 0.959 & 31.2 \\
Distributional-only & 85.8 & 0.960 & 30.5 \\
Random simplex & 86.1 & 0.959 & 30.8 \\
Uniform weights & 86.6 & 0.961 & 30.0 \\
Replay demand & 123.0 & 0.957 & -- \\
\bottomrule
\end{tabular}
\end{table}

\begin{table}[t]
\centering
\caption{External baselines re-run in the simulator on replay demand.}
\label{tab:external-v2}
\setlength{\tabcolsep}{4pt}
\begin{tabular}{lrrr}
\toprule
Method & Wait (s) & Completion & Vs. replay (\%) \\
\midrule
Oracle MPC & 91.3 & 0.960 & 26.3 \\
Scenario chance MPC & 92.2 & 0.960 & 25.6 \\
Share-target LP & 92.2 & 0.960 & 25.5 \\
Spatial-gate share LP & 92.5 & 0.960 & 25.3 \\
GPR chance MPC-lite & 94.4 & 0.959 & 24.2 \\
Wen-style rebalancing & 100.1 & 0.959 & 18.7 \\
Batch replay & 123.0 & 0.957 & -- \\
\bottomrule
\end{tabular}
\end{table}

\begin{table}[t]
\centering
\caption{Objective-weight sweep. Weights are demand/spatial/queue.}
\label{tab:sweep-v2}
\setlength{\tabcolsep}{4pt}
\begin{tabular}{lrrr}
\toprule
Config & Weights & Demand $\rho$ & Operational $\rho$ \\
\midrule
no\_queue & 0.30/0.70/0.00 & 0.938 & 0.469 \\
spatial\_heavy & 0.15/0.70/0.15 & 0.936 & 0.451 \\
current & 0.25/0.55/0.20 & 0.935 & 0.448 \\
demand\_heavy & 0.50/0.35/0.15 & 0.932 & 0.446 \\
queue\_heavy & 0.20/0.40/0.40 & 0.939 & 0.443 \\
balanced & 0.34/0.33/0.33 & 0.935 & 0.443 \\
\bottomrule
\end{tabular}
\end{table}

\begin{figure}[t]
\centering
\includegraphics[width=0.78\linewidth]{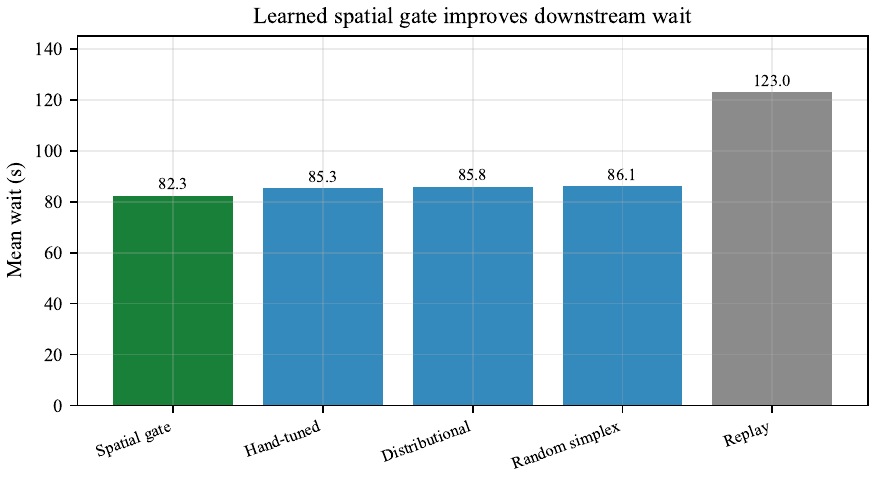}
\caption{Main learned-gate comparison over eight scenarios and ten seeds.}
\label{fig:gate-results-v2}
\end{figure}

\begin{figure}[t]
\centering
\includegraphics[width=0.78\linewidth]{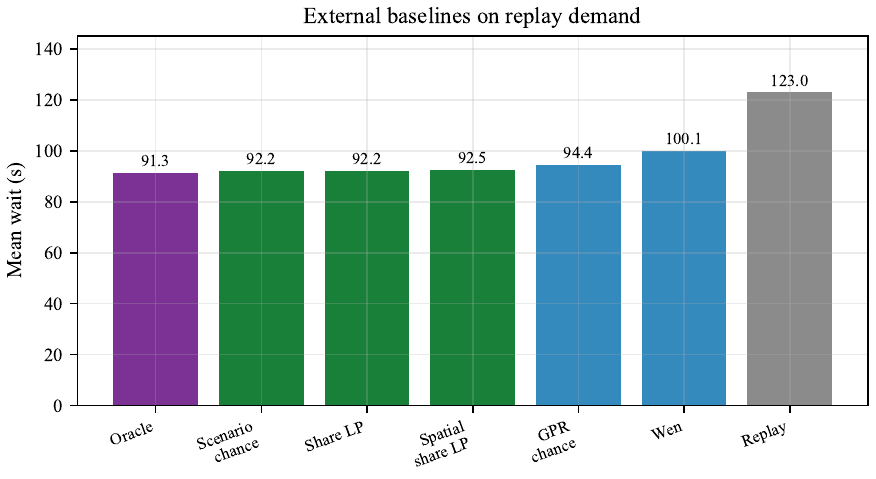}
\caption{Replay-demand external baseline comparison. Oracle MPC peeks at future
requests and is included as an upper-bound diagnostic.}
\label{fig:external-results-v2}
\end{figure}

\subsection{Gate Result}

The spatial learned gate is the best downstream method in the 10-seed gate
validation: 82.3s mean wait versus 85.3s for hand-tuned similarity and 85.8s
for distributional-only matching. Paired Wilcoxon tests over 80 scenario-seed
pairs show lower wait than hand-tuned similarity ($p=0.0036$),
distributional-only matching ($p=0.0151$), and a lower mean wait than
random-simplex weights with borderline paired significance ($p=0.0595$).
The improvement is heterogeneous across scenarios: the gate is strongest on the
New Year's Eve scenarios and morning commute, while simpler distributional or
hand-tuned similarities remain competitive on some weekend and evening blocks.
We therefore interpret the learned gate as a modest but statistically supported
operational improvement, not as a uniformly dominant retrieval rule.

\subsection{External Baseline Result}

The share-target LP fixes the earlier controller weakness. On replay demand,
the scenario chance-MPC analog and share-target LP both achieve 92.2s mean
wait, improving on Wen-style historical rebalancing at 100.1s. The reduced GPR
chance-MPC comparator is intermediate at 94.4s. Oracle MPC remains lower at
91.3s because it peeks at future requests; this makes it an upper-bound
diagnostic rather than a deployable comparator.
The contextual-DQN warm start improves replay but trails the OR baselines, so
it is reported as a time-boxed baseline rather than a contribution.
The hand-tuned and spatial-gate priors are effectively tied inside the
share-target LP (92.2s vs. 92.5s), which indicates that the controller
formulation matters at least as much as the final learned retrieval weights in
the replay-demand comparison.

\subsection{Objective Sweep}

The objective sweep is a useful caution. A no-queue objective has the best
holdout retrieval proxy, but its 10-seed downstream wait is 83.1s, worse than
the current spatial gate at 82.3s. The final selection criterion is therefore
simulator wait under frozen seeds, not holdout Spearman alone.

\section{Limitations}
\label{sec:limitations-v2}

The theory uses a stable queueing approximation and a local allocator
sensitivity assumption. These are appropriate for explaining the mechanism but
do not replace simulation. In particular, the regret result is a horizon-local
statement, not a closed-loop dynamic guarantee; the allocator-sensitivity term
is a local continuous-LP sensitivity condition with finite-fleet rounding
absorbed into the residual; and the spatial Lipschitz term is an explicit
operational assumption supported by diagnostics rather than a universal
property of all dispatch systems. The contextual-DQN row is a lightweight
time-boxed baseline, not a claim about the best possible RL policy. The charging
extension uses stylized battery dynamics and fixed charger zones. Finally, the
data are public taxi trips, not proprietary ride-hailing or real AMoD fleet
logs.

\section{Conclusion}
\label{sec:conclusion-v2}

This study frames regime retrieval as a predict-then-optimize input for fleet
repositioning. The method learns spatially regularized regime weights, connects
calibration error to wait through a stable-queue surrogate regret decomposition,
and compares against external-style baselines re-run in the same simulator. The
empirical gains are moderate for learned retrieval but robust across the
10-seed validation, while the share-target LP provides a simple controller that
improves on Wen-style historical rebalancing on replay demand.

\appendix

\section{Supporting Baselines and AMoD Extension}
\label{app:supporting-v2}


\begin{table}[t]
\centering
\caption{Per-scenario learned-gate wait times over 10 seeds.}
\label{tab:gate-scenario-v2}
\footnotesize
\setlength{\tabcolsep}{4pt}
\begin{tabular}{lrrrrr}
\toprule
Scenario & Replay & Learned & Hand & Dist. & Random \\
\midrule
jan\_nye\_am & 187.5 & 145.0 & 156.7 & 163.3 & 157.0 \\
jan\_nye\_pm & 152.4 & 72.5 & 91.5 & 89.3 & 102.9 \\
jan\_weekday\_am & 106.2 & 70.6 & 73.4 & 74.1 & 74.3 \\
jan\_weekday\_pm & 106.9 & 75.5 & 80.9 & 66.4 & 74.9 \\
jan\_weekend\_mid & 96.2 & 81.2 & 63.6 & 83.0 & 68.4 \\
jun\_late\_night & 127.9 & 64.3 & 69.7 & 63.4 & 69.1 \\
jun\_weekday\_am & 107.0 & 72.4 & 75.4 & 70.5 & 71.5 \\
jun\_weekday\_pm & 100.0 & 77.1 & 70.9 & 76.5 & 70.6 \\
\bottomrule
\end{tabular}
\end{table}

\begin{table}[t]
\centering
\caption{Per-scenario replay-demand external-baseline wait times over 10 seeds.}
\label{tab:external-scenario-v2}
\footnotesize
\setlength{\tabcolsep}{4pt}
\begin{tabular}{lrrrrrr}
\toprule
Scenario & Replay & Wen & Chance & LP & Spatial LP & Oracle \\
\midrule
jan\_nye\_am & 187.5 & 156.0 & 153.3 & 153.4 & 153.7 & 153.0 \\
jan\_nye\_pm & 152.4 & 122.3 & 109.8 & 109.8 & 110.4 & 107.9 \\
jan\_weekday\_am & 106.2 & 85.9 & 81.3 & 80.9 & 81.0 & 79.0 \\
jan\_weekday\_pm & 106.9 & 86.7 & 78.2 & 78.1 & 78.4 & 77.3 \\
jan\_weekend\_mid & 96.2 & 79.7 & 71.4 & 71.8 & 71.7 & 71.3 \\
jun\_late\_night & 127.9 & 103.9 & 94.1 & 94.7 & 94.5 & 92.9 \\
jun\_weekday\_am & 107.0 & 86.8 & 75.6 & 75.4 & 75.9 & 75.5 \\
jun\_weekday\_pm & 100.0 & 79.5 & 73.6 & 73.5 & 74.0 & 73.4 \\
\bottomrule
\end{tabular}
\end{table}

\begin{table}[t]
\centering
\caption{Time-boxed contextual-DQN baseline smoke.}
\label{tab:dqn-v2}
\setlength{\tabcolsep}{4pt}
\begin{tabular}{lrrr}
\toprule
Method & Wait (s) & Completion & Vs. replay (\%) \\
\midrule
Contextual-DQN warm start & 112.9 & 0.958 & 8.3 \\
Batch replay & 123.2 & 0.957 & -- \\
\bottomrule
\end{tabular}
\end{table}

\begin{table}[t]
\centering
\caption{Centralized AMoD charging smoke.}
\label{tab:amod-v2}
\setlength{\tabcolsep}{4pt}
\begin{tabular}{lrrr}
\toprule
Method & Wait (s) & Completion & Low-charge/call \\
\midrule
Central share LP & 97.2 & 0.960 & 0.0 \\
Charging-aware share LP & 99.1 & 0.960 & 11.2 \\
\bottomrule
\end{tabular}
\end{table}

The contextual-DQN and AMoD rows are intentionally supporting experiments. They
answer likely reviewer questions without changing the paper's scope: RL is not
the primary contribution, and the charging model is a stylized centralized-fleet
extension rather than a complete electric-fleet model.

\bibliographystyle{plainnat}

\end{document}